\documentclass[twoside]{article}

\usepackage[accepted]{aistats2019}
\usepackage[round]{natbib}

\usepackage{microtype}
\usepackage{graphicx}
\usepackage{subfigure}
\usepackage{booktabs} % for professional tables

\usepackage{amsfonts}
\usepackage{amsmath}
\usepackage{amssymb}
\usepackage{amsthm}
\usepackage{natbib}
\usepackage{graphicx}
\usepackage{color}
\usepackage{makecell}
\usepackage[noend]{algorithm,algorithmic}
\usepackage{capt-of}

\usepackage{stmaryrd}
\usepackage[colorinlistoftodos]{todonotes}
\usepackage[colorlinks=true, allcolors=blue]{hyperref}

\newtheorem{definition}{Definition}
\newtheorem{theorem}{Theorem}

\newtheorem{lemma}{Lemma}
\newtheorem{proposition}{Proposition}
\newtheorem{corollary}{Corollary}

\newcommand{\secref}[1]{Section \ref{#1}}
\newcommand{\aaa}{\mathbf{a}}
\newcommand{\ww}{\mathbf{w}} 
\newcommand{\xx}{\mathbf{x}} 
\newcommand{\ee}{\mathbf{e}} 
\newcommand{\zz}{\mathbf{z}} 
\newcommand{\DD}{\mathcal{D}} 
\newcommand{\WW}{\mathcal{W}} 
 
\newcommand{\YY}{\mathcal{Y}} 
\newcommand{\ZZ}{\mathcal{Z}} 
\newcommand{\HH}{\mathcal{H}}  
  
\newcommand{\GG}{\mathcal{G}}
\newcommand{\AAA}{\mathcal{A}} 
\newcommand{\comment}[1]{}
\newcommand{\todoamir}[1]{{\color{red}#1}}

\newcommand{\reals}{\mathbb{R}} 
 
\newcommand{\ball}{\mathbb{B}} 
\newcommand{\poly}{\mathrm{poly}} 
 
\renewcommand{\kappa}{\mathcal{K}}
\newcommand{\inner}[1]{\langle #1 \rangle}

\usepackage[utf8]{inputenc} % allow utf-8 input
\usepackage[T1]{fontenc}    % use 8-bit T1 fonts
\usepackage{hyperref}       % hyperlinks
\usepackage{url}            % simple URL typesetting
\usepackage{booktabs}       % professional-quality tables
\usepackage{amsfonts}       % blackboard math symbols
\usepackage{nicefrac}       % compact symbols for 1/2, etc.
\usepackage{microtype}      % microtypography

\begin{document}



\twocolumn[

%\aistatstitle{Hybrid Rule-Driven and Regularized Classification}
\aistatstitle{Learning Rules-First Classifiers}

\aistatsauthor{ Deborah Cohen \And Amit Daniely \And Amir Globerson \And Gal Elidan }

\aistatsaddress{ Google Research \And Google Research \\ Hebrew University \And Google Research \\ Tel-Aviv University \And Google Research \\ Hebrew University}
]

\begin{abstract}
Complex classifiers may exhibit ``embarassing'' failures in cases where humans can easily provide a justified classification. Avoiding such failures is obviously of key importance. In this work, we focus on one such setting, where a label is perfectly predictable if the input contains certain features, or rules, and otherwise it is predictable by a linear classifier. We define a hypothesis class that captures this notion and determine its sample complexity. We also give evidence that efficient algorithms cannot achieve this sample complexity. We then derive a simple and efficient algorithm and show that its sample complexity is close to optimal, among efficient algorithms. Experiments on synthetic and sentiment analysis data demonstrate the efficacy of the method, both in terms of accuracy and interpretability.
\end{abstract}

%\todogal{I really don't like this new title - let's discuss}
%\todoamir{I also don't like it...}

\section{Introduction}
The accuracy of machine learning algorithms has dramatically improved since the re-emergence of deep learning models. However, in many machine learning applications, the model will make ``embarassing'' mistakes. Namely, mistakes on examples that a human would classify easily, and have a clear explanation for her decision. As a motivating example, consider a medical diagnosis system that, on average, performs better than the family doctor. However, every now and then, the system makes an embarrassing mistake and fails in a scenario where a simple mechanism can provide the correct diagnosis. As another example, consider an online streaming platform where it would be ``embarrassing'' not to recommend episode $i+1$ to someone who is watching episode $i$ in a series. %, particularly when $i > 1$. 

Clearly, we would like to avoid such mistakes. This is important for improving usability of learned models, and for making them more interpretable. A key challenge in addressing the above problem is defining the notion of an embarrassing mistake. From the viewpoint of standard statistical learning theory, all mistakes are identical, and one is not more embarrassing than the other. But, we can structure our hypothesis class such that ``easy'' cases are processed in an explainable way.

We take the first step toward an explicit formalization of this goal by considering easy examples to be those whose label is deterministic given certain values of a single feature (e.g., in the streaming example above, if we observed episode 3 of a series, we will want to watch episode 4). However, we clearly do not expect all samples to be classified using rules, and therefore allow the label to also result from a different classifier over the other features, when no rule applies. We call such hybrid models \emph{rules-first classifiers}. Specifically, we consider the case where for a set $\kappa$ of $k$ ``rule'' features, the label is $1$ if any feature in $\kappa$
is non-negative. Otherwise, the label is determined by a linear classifier whose norm is bounded by $B$. We call such distributions $(k,B)$-realizable.

We investigate the computational and sample complexity of learning $(k,B)$-realizable distributions, and contrast these with related hypothesis classes defined by a bound on $\ell_1$ or $\ell_2$ norms. Specifically, we prove that the sample complexity of the problem is $\tilde \Theta \left( \frac{k+B^2}{\epsilon} \right)$. Interestingly, we show that this sample complexity is substantially better compared to that of the natural convex relaxation, which is $\tilde \Omega \left(\frac{kB^2}{\epsilon}\right)$.

After settling the statistical complexity for the problem, we investigate its computational complexity. We derive an efficient greedy algorithm for the classification task, and show that it enjoys a sample complexity of $\tilde\Theta\left(\frac{kB + B^2}{\epsilon}\right)$. While this sample complexity is much better compared to the natural convex relaxation, it is still inferior to the information theoretic limit of $\tilde \Theta \left( \frac{k+B^2}{\epsilon} \right)$.

Can better sample complexity be achieved by efficient algorithms? We give evidence that the answer is negative. Indeed, we show that an efficient algorithm whose sample complexity is better than $\tilde\Theta\left(\frac{kB + B^2}{\epsilon}\right)$ would lead to efficient algorithms for problems that are hypothesized to be hard.

The topic of rule learning has been studied in the past of course \citep[e.g., see][and \secref{sec:related} for more details]{rivest1987learning,zhang2002association}. Most of these approaches consider the case where \emph{every} classification decision corresponds to activating a rule (e.g., for decision trees). Here we focus on the arguably more realistic setting whereby rules only apply to a subset of the cases, i.e. the easily explained examples, and other cases are covered by a function of all the features. To the best of our knowledge, we provide the first theoretical characterization of this rules-first setting.

%The paper is organized as follows. Section \ref{sec:preliminaries} provides some necessary preliminaries and Section \ref{sec:formulation} formulates the problem of learning in the presence of rules. The sample complexity of our proposed hypothesis classes are derived in Section \ref{sec:sample_comp}, and efficient algorithms are presented in \ref{sec:algos}. In Section \ref{sec:hardness}, we provide evidence that these achieve the minimal sample complexity required for poly-time algorithms. Experimental results comparing our approach to traditional $\ell_1$ and $\ell_2$ regularizations are reported in Section \ref{sec:sims}. Proofs omitted from the text are provided in the appendix.
\section{Preliminaries}
\label{sec:preliminaries}
We begin with notation and relevant background.
Throughout the paper, the following notations are used. The set of integers in $[1,d]$ is denoted by $[d]$ and the complement of a set $\kappa$ is $\kappa^c$. We denote column vectors by boldface letters. The  $j$th feature of $\xx$ is denoted $\xx(j)$. The vector $\zz$ restricted to the set $\kappa$ is $\zz_{| \kappa}$ and $\inner{\cdot, \cdot}$ stands for the inner product. For $k\le d$, we denote $\binom{[d]}{k} = \{A\subseteq [d] \mid |A|=k\}$ and $\binom{[d]}{\le k} = \{A\subseteq [d] \mid |A|\le k\}$. The $\ell_1$, $\ell_2$ and $\ell_{\infty}$ norms are $\|\cdot\|_1$, $\|\cdot\|_2$ and $\|\cdot\|_\infty$, respectively. We also use the $\ell_0$-pseudonorm  $\|\xx\|_0 = |\{j \mid \xx(j)\ne 0\}|$. \comment{\todoamir{Should be ``norm''. Debby:changed to pseudonorm}}
We denote $\ball_B^{d,p} = \{\xx\in\reals^d \mid \|\xx\|_p\le B\}$.
For $p=2$, let $\ball_B^{d}=\ball_B^{d,2}$.
%The expectation of random variables is $\mathbb{E}$ and $\Pr$ denotes the probability of an event.

\paragraph{\bf Regularized Linear Classification Models:}
Consider the standard supervised classification problem. Let
$S=\{(\xx_i,y_i)\}_{i=1}^m$ be a set of $m$ training samples, drawn i.i.d. from some distribution $\DD$ over $\mathcal{X} \times \mathcal{Y}$, with $\mathcal{X}=\mathbb{R}^d$ and $\mathcal{Y}=\{\pm1\}$. 
To avoid measure theoretic subtleties, we assume that the support of $\DD$ is finite (none of the results will depend on its cardinality and all will hold for $\DD$ with infinite support). The goal is to find a classifier $h:\mathcal{X}\to\mathcal{Y}$ whose error $\mathrm{Err}_\mathcal{D}(h) = \Pr_{(\xx,y)\sim\mathcal{D}}\left(h(\xx)\ne y\right)$ is as small as possible. 

We consider classes of linear classifiers. Namely, classes of the form $\HH = \{\xx \mapsto \langle \ww,\xx \rangle | \ww\in \mathcal{W}\}$, for some $\mathcal{W}\subset \reals^d$. Two typical choices of $\WW$ are the $\ell_1$ and $\ell_2$ balls, as well as combinations thereof such as the elastic-net ball \citep{zou2005regularization} where $\left\{ \ww : 
\lVert\ww\rVert_1 \leq B_1 \text{ and } \lVert\ww\rVert_2^2 \leq B_2^2\right\}$. Recently, \citet{zadorozhnyi2016huber} also proposed the Huber-norm ball $\left\{ \ww_a + \ww_b : \lVert\ww_a\rVert_1 \leq B_1 \text{ and } \lVert\ww_b\rVert_2^2 \leq B_2^2\right\}$ for $B_1, B_2>0$.
\comment{
Traditionally, regularization, such as the $\ell_2$-norm penalty, which encourages the fitted parameters to be small, aims at preventing overfitting. The $\ell_1$ penalty further favors sparse parameter vectors. The  elastic-net and Huber-norm penalties allow a combination of dense and sparse weight vectors \citep{zou2005regularization, zadorozhnyi2016huber}.
}

A popular approach for the classification problem is to minimize a surrogate loss function. Namely, given a class $\HH$ of functions from $\reals^d$ to $\reals$ and a loss function $\ell:\reals\times \YY \to[0,\infty)$, solve $\min_{h\in \HH} \ell(h,S)$, where
$\ell(h,S) = \frac{1}{m}\sum_{i=1}^m \ell(h(\xx_i),y_i)$.
The expected loss $\ell$ of $h\in\HH$ with respect to $\DD$ is
$\ell(h,\DD) = \mathbb{E}_{(\xx,y)\sim\DD} \ell(h(\xx),y)$.
The optimal true loss is $\ell(\HH,\DD) = \inf_{h\in \HH}\ell(h,\DD)$, and the optimal empirical loss is $\ell(\HH,S) = \inf_{h\in \HH}\ell(h,S)$.

Popular loss functions are the mis-classification (zero-one) loss:
\begin{equation*}
\ell_{\text{mis}}(\hat y, y) \triangleq \begin{cases}
0 & \hat y \cdot y > 0 \\
1 & \hat y \cdot y \le 0,
\end{cases}
\end{equation*}
%\begin{equation*}
%\ell_{\text{margin}}(\hat y, y) \triangleq %\begin{cases}
%0 & \hat y \cdot y > 1
%\\
%1 & \hat y \cdot y \le 1,
%\end{cases}
%\end{equation*}\end{tabular},
the margin loss where the above 0 threshold is relaxed to 1, the hinge loss
%\begin{equation*}
$\ell_{\text{hinge}}(\hat y, y) \triangleq \max \{0, 1 - \hat y \cdot y \}$,
%\end{equation*}
and the ramp loss $\ell_{\text{ramp}}(\hat y, y) \triangleq \llbracket 1 - \hat y \cdot y \rrbracket$, where $\llbracket r \rrbracket \triangleq \max(0, \min(r,1))$.
%\[
%\llbracket r \rrbracket \triangleq 
%\begin{cases}
%1 & r\ge 1,
%\\
%r & 0\le r \le 1,
%\\
%0 & r \le 0.
%\end{cases}
%\]

Note that the ramp loss is upper-bounded by the margin loss and lower-bounded by the mis-classification error. Therefore, whenever $\DD$ has a low large-margin loss, it also has low ramp loss. Likewise, once we find a hypothesis with small ramp loss, we also find a hypothesis with small mis-classification loss.

\paragraph{\bf Sample Complexity Definitions:}
We now define the sample complexity of an algorithm and a hypothesis class with respect to a loss function, which we use later on to evaluate and compare different algorithms.
\begin{definition} [{\bf Sample Complexity of Algorithm}]
Fix a hypothesis class $\HH$.
The sample complexity of an algorithm $\AAA$ is the function $m_\AAA:(0,1)\times (0,1)\to \mathbb N$ so that $m_\AAA(\epsilon,\delta)$ is the minimal number for which the following holds: If $m\ge m_\AAA(\epsilon,\delta)$, then w.p.\ $\ge 1-\delta$ over the choice of $S$ and the internal randomness of $\AAA$, we have that $\ell_{\text{ramp}}(\AAA(S),\DD)\le \ell_{\text{ramp}}(\HH,\DD)+\epsilon$.
\end{definition}

\begin{definition} [{\bf Sample Complexity of Hypothesis Class}]
Fix a hypothesis class $\HH$ and a loss $\ell$.
The sample complexity of $\HH$ is  $m_\HH(\epsilon,\delta) = \min_{\AAA}m_\AAA(\epsilon,\delta)$.
\end{definition}
We say that $\DD$ is realizable if $\ell_{\text{ramp}}(\HH,\DD)=0$. Likewise, $\DD$ is $\eta$-realizable if $\ell_{\text{ramp}}(\HH,\DD)\le \eta$.
The realizable sample complexity of an algorithm and a class is defined similarly to the standard sample complexity, but restricted to realizable $\DD$. We note that our definitions of sample complexity consider the ramp loss. This is motivated by the properties of the ramp loss noted above. From now on, we fix $\delta$ to be a small constant, and omit it from the complexity measures.
\section{The Rules-First Learning Problem}
\label{sec:formulation}
We are now ready to formalize our learning problem. Recall that we would like to learn rules-first classifiers. Namely, classifiers whose outcome is either determined via a small set of features, which are referred to as rules, or a bounded norm linear classifier on the remaining features. A simple such rule based case is when we have a set $\kappa$ of $k$ features such that the label is $1$ if one of these features is positive, i.e., 
\begin{equation}
\Pr_{(\xx,y)\sim\DD}\left(y =1 \mid \xx(j) > 0\text{ for some }j\in\kappa\right)=1,
\label{eq:rule_predict}
\end{equation}
and otherwise the label is determined by a bounded norm linear classifier, i.e.,
\begin{equation}
\Pr_{(\xx,y)\sim\DD}\left(y\inner{\ww, \xx} \ge 1 \mid \xx(j) \le 0\text{ for all }j\in\kappa\right)=1,
\label{eq:linear_predict}
\end{equation}
where $\ww\in \reals^d$ with $\|\ww\|_2^2\le B^2$.  

\begin{definition}
A distribution $\DD$ is $(k,B)$-realizable if there is a set $\kappa 
\in \binom{[d]}{\le k}$ and a weight vector $\ww$ for which \eqref{eq:rule_predict} and \eqref{eq:linear_predict} hold.
\end{definition}
An equivalent notion with $\ell_1$ regularization over $\ww$ may be defined, such that the results presented in the following sections transfer in the expected way.

In the above definition, a single rule can determine the label. We next consider a broader set of distributions, which we will use for deriving the sample complexity of $(k,B)$-realizable $\DD$. Begin by noting that if $\DD$ is $(k,B)$-realizable then there are vectors $\ww_a$ with $\|\ww_a\|_2^2\le B^2$ and $\ww_b$ with $\|\ww_b\|_0\le k$ such that:\footnote{To see this, note that one can take $\ww_a = \ww$ and $\ww_b$ to be the indicator vector of $\kappa$, multiplied by a large enough scalar since we assume that $\DD$ has finite support.}
\begin{equation}\label{eq:weakly_realizable}
\Pr_{(\xx,y)\sim\DD}\left(y\inner{\ww_a + \ww_b, \xx} \ge 1\right)=1.
\end{equation}
%\todo{Not clear why we call the case below weak realizability. Also, it's not stated realizability in what class. Why not say that if $D$ is $(k,B)$ realizable then it is also $\HH_{2,0}$ realizable?} 
Motivated by this observation, we say that $\DD$ is $(k,B)$-weakly realizable if there exist norm bounded $\ww_a,\ww_b$ as above, such that \eqref{eq:weakly_realizable} holds. A $(k,B)$-weakly realizable distribution can be realized by the following hypothesis class (we omit the dependence on $k,B$):
{\small{\begin{equation} \label{eq:H20}
\HH_{2,0} = \{\xx\in \ball_1^{d,2} \mapsto \langle \ww_a+\ww_b,\xx \rangle \mid \lVert \ww_a\rVert_2^2 \leq B^2, \lVert \ww_b\rVert_0 \leq k \}.
\end{equation}}}
Namely, $\DD$ is realizable by $\HH_{2,0}$ if and only if it is $(k,B)$-weakly realizable. The hypothesis class $\HH_{2,0}$  induces weight vectors composed of $k$ unbounded entries (rules) and a remaining $d-k$ entries with bounded $\ell_2$ norm. This drives the prediction to be dictated by the $k$ features with highest weights, or rules, and in their absence, to be determined by a bounded linear classifier on the remaining features. Similarly to $\HH_{2,0}$, we define:
{\small{\begin{equation} \label{eq:H10}
\HH_{1,0} = \{\xx\in \ball_1^{d,\infty} \mapsto \langle \ww_a+\ww_b,\xx \rangle \mid \lVert \ww_a\rVert_1 \leq B, \lVert \ww_b\rVert_0 \leq k \}.
\end{equation}}}

\vspace{-0.5cm}
As we shall see, these \emph{rules-first learning} formulations lead to sample complexity reduction as well as practical advantages. Specifically, the contributions of this work are as follows (ignoring logarithmic factors):
\begin{itemize}
\item We show that the sample complexity of $(k,B)$-realizable distributions\comment{$(k,B)$-weakly realizable distributions} is $\frac{k+B^2}{\epsilon}$.
\item We derive an efficient and simple greedy algorithm for learning $(k,B)$-realizable distributions, with somewhat inferior sample complexity of $\frac{Bk+B^2}{\epsilon}$.
\item We give evidence that the sample complexity of our greedy algorithm is close to optimal among efficient algorithms and show that it is better than that of the natural convex relaxation of the problem.
\item We experiment with algorithms for the aforementioned scenario, comparing the greedy approach to the traditional $\ell_1$ and $\ell_2$ regularization approaches.
\end{itemize}
Taken together, our results indicate that the problem of learning rules-first classifiers exhibits an interesting statistical computational trade-off, and that efficient algorithms work well in practice.
\section{Sample Complexity}
\label{sec:sample_comp}

In this section, we derive the sample complexity of the rule-based hypothesis classes $\HH_{2,0}$ and $\HH_{1,0}$ and use the former to obtain the sample complexity of $(k,B)$-realizable distributions.

\begin{theorem}\label{thm:sample_complexity_20}
The sample complexity of $\HH_{2,0}$ is $\tilde{O} \left( \frac{k \log d + B^2}{\epsilon^2}\right)$.
\end{theorem}

\begin{theorem}\label{thm:sample_complexity_10}
The sample complexity of $\HH_{1,0}$ is $\tilde{O} \left( \frac{k \log d + B^2 \log d }{\epsilon^2}\right)$.
\end{theorem}

To prove Theorem \ref{thm:sample_complexity_20}, we rely on the following result from \citet{sabato2013distribution}, which considers the problem of distribution-dependent sample complexity. In their setting, the distribution of the input features has few directions in which the variance is high, but the combined variance in all other directions is small. With this assumption, they show that the sample complexity is characterized by the sum of the number of high-variance dimensions $k$ and the squared norm in the other directions $B^2$. 

Formally, for $\kappa\in \binom{[d]}{k}$, let
\[
\ZZ_{\kappa,B} = \{ \zz\in \reals^d \mid \|\zz_{|\kappa^c}\|_2^2\le B^2\}.
\]
Consider the class
\begin{equation}\label{eq:sivan_class}
\HH_{\kappa,B} = \{ \xx \in \ZZ_{\kappa,B} \mapsto \inner{\ww,\xx} 
\mid \ww\in \ZZ_{\kappa,1}\}.
\end{equation}
Then, \citet{sabato2013distribution} show the following result.
\begin{proposition} \label{prop:sivan}
For any $\delta \in (0,1)$, with probability $1-\delta$, every $h \in \HH_{\kappa,B}$ satisfies
\begin{multline} \label{eq:gen_bound_mis_margin}
\ell_{\text{ramp}}(h,\DD) \leq \\
\ell_{\text{ramp}}(h,S) + \sqrt{\frac{O\left(k + B^2\right) \ln(m) }{m}} + \sqrt{\frac{8 \ln (2/\delta)}{m}}.
\end{multline}
\end{proposition}
The above result focuses on the class $\HH_{\kappa,B}$ which makes an assumption on the input features $\xx$. In our setting, we make a similar distributional assumption but on the conditional distribution of the \emph{target variable} given the special set of features, i.e. the rules. Specifically, we wish to derive sample complexity bounds for the rule-based hypothesis classes \eqref{eq:H20} as well as \eqref{eq:H10}.

To do so, we associate each example $\xx \in \ZZ_{B,\kappa}$ with the example $\xx' \in \ZZ_{1,\kappa}$ obtained by dividing each coordinate $i\in \kappa$ by $B$. We then have that the sample complexity of $\HH_{\kappa,B}$ is the same as that of the class
\begin{equation}\label{eq:our_class}
\vspace{-0.25cm}
\GG_{\kappa,B} = \{ \xx \in \ZZ_{\kappa,1} \mapsto \inner{\ww,\xx} 
\mid \ww\in \ZZ_{\kappa,B}\}.
\end{equation}

Now, since $\HH_{2,0}\subset \cup_{\kappa\in \binom{[d]}{k}} \left(  \GG_{\kappa,B}\right)|_{\ball^d}$, Proposition \ref{prop:sivan} and a union bound imply Theorem \ref{thm:sample_complexity_20}. A detailed proof as well as an adaptation for Theorem \ref{thm:sample_complexity_10} are provided in the supplementary materials.

We note that both theorems are tight, up to logarithmic factors. Indeed, both $\HH_{2,0}$ and $\HH_{1,0}$ realize the class of $k$-disjunctions, which has sample complexity $\Omega\left( \frac{k}{\epsilon^2}\right)$. Likewise, $\HH_{2,0}$ (respectively $\HH_{1,0}$) contains the class of linear classifiers with $\ell_2$ norm (respectively $\ell_1$ norm) smaller than $B$,  which has sample complexity $\Omega\left( \frac{B^2}{\epsilon^2}\right)$ \citep{anthony2009neural}. Hence, both rule-based classes have sample complexity of $\Omega\left( \frac{B^2+k}{\epsilon^2}\right)$.

We also note that boosting \citep{freund1997decision}  implies that the realizable sample complexities of $\HH_{2,0}$ and $\HH_{1,0}$ are $\tilde{O} \left( \frac{k \log d + B^2}{\epsilon}\right)$ and $\tilde{O} \left( \frac{k \log d + B^2\log d}{\epsilon}\right)$, respectively. Indeed, once we fix $\epsilon$, the general sample complexity result yields a weak learner with sample complexity of $\tilde{O}\left(k + B^2\right)$. Applying boosting on top of it yields a strong learner with the above mentioned sample complexity guarantees in the realizable case.

As a corollary to Theorem \ref{thm:sample_complexity_20} we obtain the sample complexity of learning $(k,B)$-realizable distributions. This follows from the equivalence of weak $(k,B)$-realizability to learning in $\HH_{2,0}$, and the fact that $(k,B)$-realizability implies weak $(k,B)$-realizability.
\begin{corollary}\label{thm:sample_complexity_kb}
The sample complexity of $(k,B)$-realizable distributions is $\tilde{O} \left( \frac{k \log d + B^2}{\epsilon}\right)$.
\end{corollary}
\section{Efficient Algorithms}
\label{sec:algos}

\vspace{-0.2cm}

The sample complexity obtained in the previous section may be achieved by using an ERM algorithm. Unfortunately, in the sequel we argue that it is unlikely that there is an efficient implementation of such an algorithm. Thus, we begin by proposing an efficient learning procedure, and provide corresponding sample complexity results. Further, we show our proposed algorithm dominates the natural regularization based approach to the problem.

\subsection{An Efficient Greedy Algorithm}
We start with the description of a greedy based algorithm and analysis of its sample complexity. Let $S$ be a training sample. A {\em rule} is a coordinate $j\in [d]$ such that $y_i=1$ whenever $\xx_i(j)>0$. We say that a rule $j$ {\em covers} an example $(\xx,y)$ if $\xx(j)> 0 $.
Consider the {{\textit{GreedyRule}}} algorithm in Figure \ref{fig:greedy}. Defining a similar algorithm with $\ell_1$ regularization over $\ww$ is straightforward.
\begin{algorithm}[ht]
  \begin{algorithmic}
    \STATE Initialize $\kappa = \emptyset$ and let $S_{\text{non-covered}} = S$.
    \WHILE {there is a rule $j\in [d]$ that covers more than $\frac{m}{100k(B+1)}$ examples from $S_{\text{non-covered}}$}
    \STATE \ 
    \STATE - Add $j$ to $\kappa$
    \STATE - Discard samples covered by $j$ from $S_{\text{non-covered}}$
    \ENDWHILE
    \STATE Find $\ww$ that minimizes the hinge loss on $S_{\text{non-covered}}$ such that $\|\ww\|_2^2\le B^2$
  \captionof{figure}{The {{\textit{GreedyRule}}} Algorithm}\label{fig:greedy}
  \end{algorithmic}
  \label{alg:greedy}
\end{algorithm}

Now define {{\textit{BoostRule}}} to be a boosting algorithm that uses {{\textit{GreedyRule}}} as a weak learner.

\begin{theorem}
BoostRule can learn $(k,B)$-realizable distributions with a sample complexity of $\tilde O\left(\frac{kB+B^2}{\epsilon}\right)$.
\end{theorem}

We will prove this by showing, in the following lemma, that {{\textit{GreedyRule}}} (Figure \ref{fig:greedy}) is a weak learner. Namely, it is guaranteed to return a hypothesis with error $\frac{1}{2} - \Omega(1)$ whenever it runs on $(k,B)$-realizable distributions. The theorem will then be implied by boosting \citep{freund1997decision}. Indeed, applying boosting on top of a weak learner with sample complexity of $M$ results in a strong learner with sample complexity of $\tilde{O}\left(\frac{M}{\epsilon}\right)$.

\begin{lemma}
If $\DD$ is $(k,B)$-realizable and $m = \tilde\Omega\left(kB+B^2\right)$, then w.h.p.\ the greedy algorithm will return a hypothesis with error $\le 1/4$.
\end{lemma}

\begin{proof} (sketch)
We first note that upon termination of the algorithm, $\kappa$ contains at most $100k(B+1)$ rules. Hence, the hypothesis returned by the algorithm belongs to $\HH_{2,0}$ with $100k(B+1)$ instead of $k$. By Theorem \ref{thm:sample_complexity_20} and the assumption that $m = \tilde\Omega\left(kB+B^2\right)$, it is enough to show that the empirical error is $\le 1/5$. Indeed, for this amount of examples, Theorem \ref{thm:sample_complexity_20} guarantees a generalization error smaller than $\frac{1}{4}-\frac{1}{5}$.

Since there are no mistakes on the covered examples, it is enough to show that at most $0.2m$ of the non-covered examples are mis-classified by the vector $\ww$ that was found in step 3. We will show an even stronger property. Namely, that
\[
\sum_{(\xx_i,y_i)\in S_{\text{non-covered}}}l_{\text{hinge}}(\inner{\ww,\xx_i},y_i) \le 0.2m.
\]
Let $\kappa^*\in \binom{[n]}{\le k}$ and $\ww^*$ be respectively a set and a vector given which $S$ is $(k,B)$ realizable. It is enough to show that
\[
\sum_{(\xx_i,y_i)\in S_{\text{non-covered}}}l_{\text{hinge}}(\inner{\ww^*,\xx_i},y_i) \le 0.2m.
\]
To see that the last equation holds, let $S^*_{\text{covered}}$ be the examples in $S$ that are covered by the rules in $\kappa^*$. 
Denoting $U = S_{\text{non-covered}}\cap S^*_{\text{covered}}$,
%and using $l_{\text{hinge}}$ for compactness,
we have that
\small{
\begin{eqnarray*}
\sum_{\substack{(\xx_i,y_i)\in \\ S_{\text{non-covered}}}}\hspace{-0.15in} l_{\text{hinge}}(\inner{\ww^*,\xx_i},y_i)
&
=
&
\sum_{(\xx_i,y_i)\in U} \hspace{-0.1in} l_{\text{hinge}}(\inner{\ww^*,\xx_i},y_i) \\ 
\le
|U|(\|\ww^*\|+1) 
& \le & |U|(B+1),
\end{eqnarray*}}
The first equality follows from the fact that, since $S$ is $(k,B)$ realizable for $\kappa^*$ and $\ww^*$, then there are no mistakes in $S^*_{\text{non-covered}}$. In other words, the only mistakes in $S_{\text{non-covered}}$ are in $U$. The result follows by noting that $|U|\le \frac{m}{100 (B+1)}$, since each rule in $\kappa^*$ covers at most $\frac{m}{100 (B+1) k}$ examples from $S_{\text{non-covered}}$, or step 2 would not terminate.
\end{proof}

\subsection{Theoretical Limitation of Regularization-based Approaches}
An alternative approach to efficiently learning sparse classifiers is to replace the sparsity (i.e., $\ell_0$) constraint with an $\ell_1$ constraint, and show that the distribution at hand can be realized by low-norm linear classifiers \citep{ng2004feature}. This suggests that we can try and learn $(k,B)$ distributions by optimizing over $\HH_{2,0}$ with the $\ell_0$ norm replaced by $\ell_1$. Refer to this class as $\HH_{2,1}$.  The following lemma proves that this strategy is inferior to the greedy algorithm. Specifically, it results in lower bounded sample complexity $\Omega(\frac{kB^2}{\epsilon})$, which is larger than the upper bound on the greedy sample complexity.

To show that an algorithm has sample complexity of at most $\frac{C}{\epsilon}$, it suffices to show that there exists a distribution which can be realized by a linear classifier of squared norm at most $C$. Namely, there is a $C$-bounded norm linear function that is greater than $1$ on positive points and smaller than $-1$ on negative points.

\begin{lemma} \label{lemma:mixed}
Let $B\ge 1$. There exists a $(k,B)$-realizable distribution $\mathcal{D}$ such that
\vspace{-0.3cm}
\begin{enumerate}
    \item The marginal distribution of $\mathcal{D}$ on $\reals^d$ is supported in $\ball^{d,2}_1$ (and hence also in $\ball^{d,\infty}_1$).
    \vspace{-0.2cm}
    \item Any linear classifier that realizes $\mathcal{D}$ with margin has squared $\ell_1$ norm $\Omega (k^2B^2)$ and squared $\ell_2$ norm $\Omega(kB^2)$.
    \vspace{-0.3cm}
\end{enumerate}
\end{lemma}
\begin{proof}
Let $d = k + B^2$ and let $\aaa = \frac{1}{B}\sum_{i=1}^{B^2}\ee_{k+i}$.
Consider the uniform distribution on
%
% \begin{multline*}
\begin{equation*}
 \frac{\ee_1 + \aaa}{\sqrt{2}},1),
 \ldots,\frac{\ee_k + \aaa}{\sqrt{2}},1), (\ee_{k+1},-1),\ldots,(\ee_{k+B^2},-1),
 \end{equation*}
% \end{multline*}
Clearly, the distribution is $(k,B)$-realizable. Likewise, if $\ww$ realizes $\mathcal{D}$, we must have $\ww_i\le -1$ for any $i=k+1,\ldots,k+B^2$. It follows that $\inner{\ww,\aaa}\le -B$. Hence, we must have $\ww_i\ge B$ for any $i=1,\ldots,k$.
\end{proof}

To conclude the lower bound argument, we note that for learning in $\HH_{2,1}$ we need to restrict the $\ell_2$ norm to at least $kB^2$ to achieve the minimal sample complexity. The latter is thus lower bounded by $\Omega(\frac{kB^2}{\epsilon})$ as the upper bound on the sample complexity with respect to the $\ell_2$ norm is tight \citep{anthony2009neural}. The sample complexity results of the above algorithmic variants are summarized in  Table~\ref{table:sample_comp}. Note that these still hold for the non-boosting version of the listed algorithms when replacing $\epsilon$ with $\epsilon^2$ in the sample complexity expressions \citep{freund1997decision}. This holds for the hardness result formulated in the next section as well.

\begin{table}[t]
\caption{Comparison of sample complexities. In the table, Comp. stands for complexity.}
\label{table:sample_comp}
\begin{center}
\begin{small}
\begin{sc}
\begin{tabular}{cccc}
\toprule
Method & Efficient & Sample Comp. \\
\midrule
ERM    & No & $\tilde{O} \left( \frac{k  + B^2}{\epsilon}\right)$\\
Convex Relaxation    & Yes & $\tilde \Omega \left( \frac{kB^2}{\epsilon}\right)$ \\
Greedy    & Yes & $\tilde O \left( \frac{kB + B^2}{\epsilon}\right)$ \\
\bottomrule
\end{tabular}
\end{sc}
\end{small}
\end{center}
\vspace{-0.7cm}
\end{table}

\subsection{Hardness}
Having shown that our greedy approach is better in terms of sample complexity than a natural regularization based approach, we now show that in some sense we cannot do better than this greedy approach. In particular, we provide evidence that its sample complexity, namely $O\left(\frac{B^2 + Bk\log(d)}{\epsilon}\right)$, is close to optimal among all efficient ($\poly(B,d)$ runtime) algorithms. Concretely, we will show that an efficient algorithm with sample complexity of $O\left(\frac{B^2 + B^{1-\alpha}k\log(d)}{\epsilon}\right)$ for any $\alpha>0$ would lead to a breakthrough in the extensively studied problem \citep[e.g., see][]{shalev2010learning, birnbaum2012learning, daniely2014complexity} of learning large margin classifiers with noise.
To do so, we require a few additional definitions. We say that a distribution $\DD$ on $\ball^{d}_1\times \{\pm 1\}$ is $(\eta,B)$-realizable if there exists $\ww\in \ball^d_B$ such that 
\[
\Pr_{(\xx,y)\sim\DD}\left(y\inner{\ww,\xx}\le 1\right) \le \eta(B).
\]
The notion of $(\eta,B)$-realizable sample is defined similarly. We next describe the problem of learning large-margin classifiers with noise rate of $\eta:\mathbb N\to [0,\frac{1}{4})$.
We are given a norm bound $B\in \{1,2,\ldots\}$ and access to an $(\eta (B),B)$-realizable distribution $\DD$ on $\ball^{B^2}_1\times \{\pm 1\}$.
The goal is to find a classifier with 0-1 error $\le \frac{1}{4}$ in time $\poly(B)$.

This problem and variants have been studied extensively. Clearly, the problem becomes easier as $\eta$ gets smaller. The best known algorithms \citep{birnbaum2012learning} can tolerate noise of rate $\frac{\poly\log(B)}{B}$. Furthermore, there are lower bounds \citep{daniely2014complexity} that show that for a large family of algorithms (specifically, generalized linear methods),  better bounds cannot be achieved. Likewise, there are hardness results \citep{daniely2016complexity} that show that, under certain complexity assumptions, no algorithm can tolerate a noise rate of $2^{-\log^{1-\alpha}(B)}$.

We will next show that algorithms for learning $(k,B)$-realizable distributions with sample complexity of $O
\left(B^2 + B^{1-2\alpha}k\log(d)\right)$ would lead to an 
algorithm for learning large margin classifiers with noise 
rate $\le \frac{1}{B^{1-\alpha}}$, improving on the current 
state of the art. By boosting, this is true even if the 
algorithm is only required to return a hypothesis with non 
trivial performance (say, error at most $0.499$) for 
$(k,B)$-realizable distributions. This serves as an indication that the sample complexity of $O
\left(B^2 + Bk\log(d)\right)$, achieved by our greedy algorithm, is close to optimal among efficient algorithms. A similar argument would rule out, under the complexity assumption from \citet{daniely2016complexity}, efficient algorithms that enjoy a sample complexity of $O
\left(B^2 + 2^{\log^{1-\alpha}(B)}k\log(d)\right)$.

We next sketch the argument. Suppose that $\AAA$ is a learner for the problem of learning  $(k,B)$-realizable distributions, with sample complexity of $O \left(B^2+B^{1-2\alpha}k\log(d)\right)$. 
Suppose now that $\DD$ is $(\eta,B)$-realizable with $\eta \le \frac{1}{B^{1-\alpha}}$ and $S$ is a sample consisting of $m = B^{10}$ points.
We will generate a new sample $\Psi(S)$ by replacing $(\xx_i,y_i)$ with $((\xx_i,\ee_i),y_i)$, where $\ee_i\in \reals^m$ is the $i$th vector in the standard basis. It is not hard to verify that, with constant probability, $\Psi(S)$ is $(2\eta m ,B)$-realizable. Indeed, the original vector that testifies that $\DD$ is $(\eta, B)$-realizable will correctly classify about $m-\eta m$ examples with margin $1$. The remaining examples can be handled using $k \approx \eta m$ rules. 
Now, since $m= \tilde{\Omega}\left( B^2+B^{1-2\alpha}\eta m \log(m)\right)$, $\AAA$ will have non-trivial performance. This translates into a non-trivial performance on the original distribution for the large margin with noise problem.

We have thus shown in the last few sections that learning with rules, while inherently hard, does lead to sample complexity improvements, and can be learned in practice using a greedy algorithm that trades-off computational and statistical efficiency. As we shall see below, learning with rules is also beneficial in practice.
\vspace{-0.6cm}
\section{Related Work \label{sec:related}}
\vspace{-0.2cm}
A long history of works in machine learning is devoted to learning rules. Association rule learning \citep{zhang2002association, agrawal1993mining} is a rule-based method for discovering relations between variables, or rules, in large databases. Rules lists \citep{rivest1987learning,clearwater1990rl4, letham2015interpretable} which consist of a series of \emph{if..., then...} statements, are a type of associative classifier, as the lists are formed from association rules. The \emph{if} statements define a partition of a set of features, or rules, and the \emph{then} statements correspond to the predicted outcome. Rules lists, or decision lists, generalize decision trees \citep{Quinlan:C45}, in the sense that any decision tree can be expressed as a decision list, and any decision list is a one-sided decision tree \citep{letham2015interpretable}.

All the above works assume that the data may be explained and \emph{perfectly} classified via a set of relatively simple rules. In contrast, we propose a hybrid and more realistic framework, where labels are determined either by a set of simple rules or by a bounded-norm classifier in examples where the rules are not applicable. To the best of our knowledge, our work is the first to investigate the computational and sample complexity of this natural setting.
 In principle, one can augment a decision tree with linear classifier nodes (i.e., oblique decision trees) to handle such cases \citep{Murthy+al:94}. However, this would result in a different linear classifier for each rule. Furthermore, learning such trees cannot be done optimally, and does not result in performance guarantees like we have here.

The works of \citep{calderon2018conditional,juba2017conditional} are concerned with a somewhat different goal as they seek to find subsets in the data, determined by k-Disjunctive Normal Form “rules” over some features, for which a good linear predictor can be found. In contrast, we aim at learning a linear classifier for all of the data unexplained by the rules. Our approach is also different from \citep{viola2001rapid} that combines increasingly more complex classifiers in a cascade in contrast to our joint learning approach. The latter work does not provide theoretical results on the proposed method. Another relevant body of work considers learning with constraints. \citet{abu1993method} deals with incorporating hints, or prior knowledge, such as invariance or oddness, in the learning process under the form of artificially generated examples. 

An alternative approach to rule learning is to consider a sparse linear classifier. Since sparsity constraints are hard to enforce, a typical approach is to use $\ell_1$ regularization as a surrogate for the $\ell_0$ sparsity constraint. Under some conditions it can be shown \citep{ng2004feature} that this may result in tractable learning of rule based classifier. Similar results are available for online learning with the Winnow algorithm and its variants \citep{littlestone1988learning}. However, these guarantees will no longer hold for the case of mixed rules and bounded norm classifier as we consider here. An additional related line of work is on mixed norm regularization (e.g., see \citep[e.g., see][]{zadorozhnyi2016huber,zou2005regularization}), which uses both norms $\ell_1$ and $\ell_2$. However, as we saw, such mixed regularization results in sample complexity bounds that are inferior to those obtained by our greedy algorithm.

\section{Experimental Evaluation}
\label{sec:sims}

We now empirically demonstrate the merit of our approach. We compare the performance of our $\ell_1$ and $\ell_2$ {\textit{GreedyRule}} to traditional $\ell_1$ and $\ell_2$ penalties. We first consider binary classification on a synthetic dataset, generated with perfect rule features, and then turn to a real-life Twitter sentiment analysis based on the SemEval '17 task \citep{SemEval:2017:task4}.

Our greedy rule-based approach, described in Section \ref{sec:algos}, iteratively  selects the feature that minimizes the current evaluation loss when added to the rules set. At each step, a regularized linear classifier is trained after removing the rule features. Prediction is then carried out first using these rules, and then by the learned classifier for examples where none of the rules apply.

For the non-rule part of our classifier, as well as the baseline classifiers, we consider a standard constrained logistic regression objective:
\begin{equation}
\min_{\ww, c} \frac{1}{m} \sum_{i=1}^m \log(\exp(-y_i(\ww^T\xx_i+c))+1) + \frac{1}{C} R(\ww),
\end{equation}
where $1/C$ is the regularization strength parameter that trades-off training accuracy and regularization and the penalty $R(\ww)$ can be either
$R(\ww) = \frac{1}{2} \lVert \ww \rVert_2^2$ or $R(\ww) = \lVert \ww \rVert_1$. We use the logistic regression implementation of the scikit-learn library \citep{scikit-learn} for both our greedy approach and the baseline linear classifier.

\subsection{Synthetic Dataset}
We generate $m$ training samples with $400$ standard features and $k=20$ rule features. The rule features are i.i.d. Bernoulli random variables with parameter $p=1/60$. The remaining features are i.i.d. random variables generated from a Gaussian distribution with $\mu=-0.02$ and $\sigma^2=1$. For each sample, $y = 1$ if one of the rule features is non zero and $y = \mbox{sign}(\inner{\ww, \xx})$ with $\ww_j =1, \text{ for } 1 \leq j \leq d$, otherwise. For the greedy algorithms, we use $2m/3$ samples for training and $m/3$ samples for evaluation to select the $k$ rule features. We then retrain the chosen classifier on the $m$ training samples. The test set is composed of $2000$ samples, generated similarly to the training samples. The results are averaged over $20$ realizations.

Figure~\ref{fig:synthetic_dataset}(left) shows the test accuracy of the different algorithms as a function of the number of training samples $m$. It can be seen that our greedy $\ell_2$ and $\ell_1$ algorithms outperform the traditional $\ell_2$ and $\ell_1$ regularized classifiers, as they succeed in finding the rule features. Appealingly, the gap between our approach and classic regularization is greater when $m$ is smaller. In Figure~\ref{fig:synthetic_dataset}(right), we show the accuracy of our greedy approaches 
as a function of $|\kappa|$, the number of rules allowed in {\textit{GreedyRule}} (Figure \ref{fig:greedy}). It can be clearly seen that increasing $|\kappa|$ towards the true $k$ improves the performance while values beyond $k$ decrease accuracy.

\begin{figure}[t!]
\centering
\begin{tabular}{cc}
%\subfigure[]{
\hspace{-0.15in}
\includegraphics[width=0.5\columnwidth]{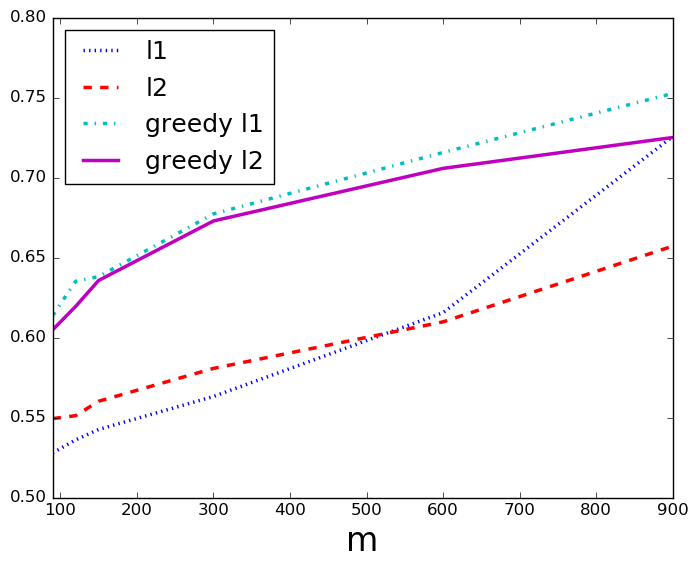} \hspace{-0.15in} &
%}
%\label{fig:synthetic_dataset_m}}
%\subfigure[]{
\includegraphics[width=0.5\columnwidth]{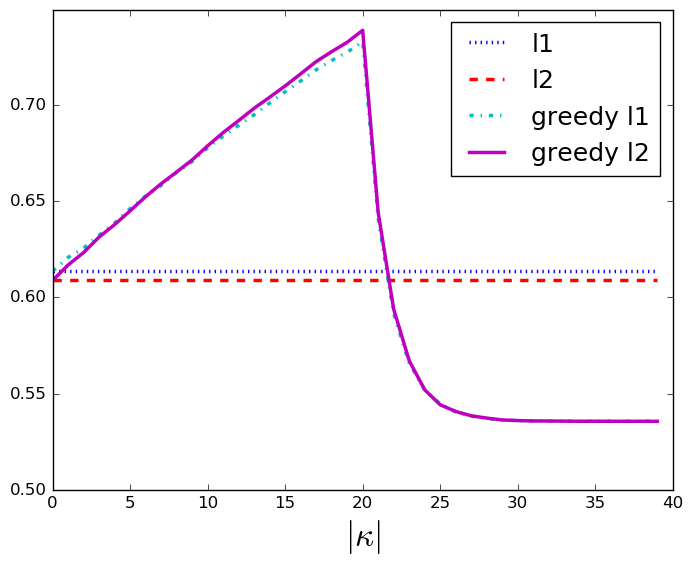}
%\label{fig:synthetic_dataset_k}}
\end{tabular}
\vspace{-0.6cm}
\caption{{\small{Accuracy comparison of the {\textit{GreedyRule}} based approach and the corresponding baseline on the synthetic dataset as a function of the number of training sample $m$ (left) and the number of rules $|\kappa|$ in the generating distribution (right).}}}\label{fig:synthetic_dataset}
\end{figure}

\begin{figure}
\centerline{
\includegraphics[width=0.75\columnwidth]{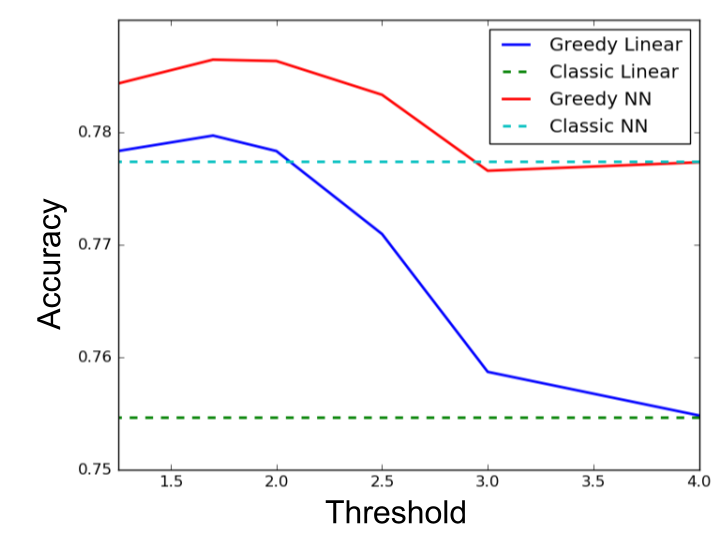}}
\label{fig:sentiment}
\vspace{-0.5cm}
\caption{{\small{Comparison of our method and the baseline classifier for two settings: using a base linear classifier or a neural network (NN). Shown is accuracy vs. the threshold used to pre-select the rule candidates.}}}
\vspace{-0.4cm}
\end{figure}

\subsection{Sentiment Analysis - Twitter}

\begin{figure*}
\begin{center}
\begin{tabular}{cc}
\begin{minipage}{0.6\columnwidth}
\label{fig:tweets}
\caption{\small Examples of tweet prediction. The last column provides the stem which was considered as a rule by the greedy algorithm. 'Pos.' and 'Neg.' stand for positive and negative, respectively.}
\end{minipage} &
\begin{tiny}
\begin{sc}
\begin{tabular}{cccc}
\toprule
Tweet & Greedy $\ell_2$ & $\ell_2$ & Stem \\
\midrule
Boko Haram on Saturday claimed responsibility \\ for attacks in Chad's capital & Neg. & Pos. & attack \\
\midrule
Dustin Johnson dealt with another major disappointment & Neg. & Pos. & disappoint \\
\midrule
Randy Orton is one of my favorites, despite everything. \\ Why? Because he's so damn good in the ring. & Pos. & Neg. & good \\
\midrule
We may believe whatever we want about gay marriage. \\ But God calls us to love, not to condemn. & Pos. & Neg. & love \\
\bottomrule
\end{tabular}
\end{sc}
\end{tiny}
\end{tabular}
\end{center}
\end{figure*}

We now turn to the SemEval-2017 Task 4 sub-task A \citep{SemEval:2017:task4} of message polarity prediction. That is, given a Twitter message, the goal is to classify whether it has a positive, negative or neutral sentiment. We note that the task cannot be reconstructed precisely since some tweets become unavailable with time. We report results on binary polarity prediction (positive vs. negative), as they better demonstrate the effectiveness of rules. Results for three way classification (not shown) were similar in trend, and resemble state-of-the-art results on this problem \citep{SemEval:2017:task4}.

The reduced dataset is composed of $3$K training tweets, $1$K evaluation tweets, and $8$K are held out as test tweets. As a pre-processing step, we clean the text by removing links and special characters. We then use the SnowBall stemmer to transform each token into a stem. We adopt a bag-of-words representation, where the features of each example are a binary vector of appearance of tokens, or stems, in the tweet. The resulting stem dictionary is constructed with respect to the training and evaluation examples and contains about $d=6.5$K tokens.

Naturally, the dataset does not contain perfect rules, which requires us to pre-select candidates that are near rules. We first discard features that appear in less than $4$ negative ($16$ positive) training tweets or for which $\hat{p}_{1,j} < 0.75$ ($\hat{p}_{0,j} < 0.9$), where $\hat{p}_{1|0,j}$ is the empirical probability that the label has value 1 or 0 given feature $x(j)$. Note that in this case, we consider rules for both labels $y=1$ and $y=0$. The discrepancy between the chosen thresholds reflect the data bias, which contains four times more positive than negative training examples. We then choose the top $k$ rules, ordered by $\sqrt{M} p_{1|0,j}$, where $M$ is the number of samples containing $x(j)$. In words, this balances between the nearness to rules and the coverage of the feature.

Figure \ref{fig:sentiment} shows the overall accuracy of the standard and greedy $\ell_2$ methods as a function of that threshold. Again, to take into account the data bias, we use the threshold as presented in the figure for rules inducing negative tweets, and four times that threshold for rules inducing positive tweets. The figure also presents the accuracy of a neural network classifier as well as a {\textit{GreedyRule}} variation using the same non-linear classifier instead of a regularized logistic regression. For both our greedy approach and the baseline non-linear classifier, we use the neural network implementation of the scikit-learn library \citep{scikit-learn} with 2 hidden layers with 5 and 2 neurons, respectively. We considered adding token pairs as features, e.g., to cope with the issue of negation. However, this did not improve the {\textit{GreedyRule}} classifier's performance.

It may be observed that, as we allow more candidate rules by lowering the threshold, accuracy improves for both linear and non-linear {\textit{GreedyRule}} classifiers. Below a certain threshold (approximately $1.5$) accuracy begins to decrease since non-rule features begin to be mistaken for rules, due to the data sparsity. We note that cross-validation on the evaluation data yields a threshold of $1.7$, which corresponds to the highest accuracy in the test set as well. Although no theoretical results have been provided for the case of non-linear classifiers, the {\textit{GreedyRule}} non-linear variation behaves similarly to its linear counterpart, as might have been expected.

The results are also quite appealing qualitatively. Stem rules chosen by the greedy algorithm have a clear sentiment semantics and include stems such as: happi, danger, evil, fail, excit, annoy, blame, loser, thank, magic, disappoint, failure, ruin, shame, stupid, love, terribl, worst, great, ridicul, disagre. Figure \ref{fig:tweets} shows a few test tweets for which our greedy $\ell_2$ linear model does well but that are misclassified by traditional $\ell_2$.

\section{Summary}
\vspace{-0.2cm}
In this work, we tackled the problem of learning rules-first classifiers. These, in addition to achieving high accuracy, do not make "embarrassing mistakes" where a simple explanation to the true label is possible, i.e., when the label can be accurately predicted from a single feature or rule. We formalized the notion of rules-based hypothesis classes, characterized the sample and computational complexity of learning with such classes and proposed an efficient greedy algorithm that trades-off computational and statistical complexity. Appealingly, its sample complexity is better than that of standard convex relaxation, and is likely optimal among all efficient algorithms. Finally, we demonstrated the benefit of our approach on simulated data as well as on a real-life tweets sentiment analysis task.

Our work is a first step toward an explicit formalization of the desideratum that the learning model does not make mistakes where good predictions can easily be achieved as well as explained. There are many intriguing directions for future developments, such as the obviously needed but non-trivial extension to soft rules and the adoption of a loss function that weights training points to account for embarrassing misclassification. More generally, we would like to learn under more flexible "embarrassment" requirements, such as ensuring the learn model does not make mistakes where simpler models do well.

\newpage
\newpage

\bibliographystyle{plainnat}
\bibliography{main}

\end{document}